\documentclass[a4paper]{article}

\usepackage{authblk}
\usepackage{graphicx}
\usepackage[colorinlistoftodos]{todonotes}

\usepackage[left=2cm, right=2cm]{geometry}

\usepackage{tikz}
\usepackage{paralist}

\usepackage{amssymb,amsmath,amsthm}
\usepackage{theoremref}
\newtheorem{definition}{Definition}

\newtheorem{proposition}{Proposition}

\title{Yes, IoU loss \textbf{is} submodular -- \emph{as a function of the mispredictions}}
\author[1]{Maxim Berman\footnote{Authors are listed in alphabetical order.}}
\author[1]{Matthew B.\ Blaschko}
\author[1,2]{Amal Rannen Triki}
\author[3]{Jiaqian Yu}

\affil[1]{ESAT-PSI, KU Leuven, Belgium}
\affil[2]{Yonsei University, Seoul, Korea}
\affil[3]{Samsung R\&D Institute China-Beijing}

\begin{document}
\maketitle

\begin{abstract}
This note is a response to \cite{arXiv:1809.00593} in which it is claimed that \cite[Proposition~11]{Yu2015b} is false.  We demonstrate here that this assertion in \cite{arXiv:1809.00593} is false, and is based on a misreading of the notion of set membership in \cite[Proposition~11]{Yu2015b}. 
We maintain that \cite[Proposition~11]{Yu2015b} is true. 
\end{abstract}

Based on the empirical risk principle, one should minimize at training time the loss that one wishes to evaluate at test time \cite{DBLP:books/daglib/0097035}.  In \cite{yu:hal-01151823,Yu2015b}, we have studied the construction of surrogates for loss functions that are submodular \emph{with respect to the set of mispredictions}.  One such example  \cite[Proposition~11]{Yu2015b} is the Intersection over Union or IoU loss, perhaps more appropriately called the Jaccard loss as it is one minus the Jaccard index~\cite{doi:10.1111/j.1469-8137.1912.tb05611.x}.  This is particularly popular in the evaluation of image segmentation algorithms, and has gained popularity in the computer vision literature due to its inclusion in the evaluation criteria for benchmark challenges \cite{Everingham15}.

In recent years, we have devoted some effort to understanding the properties of the Jaccard index and its implications for the construction of loss surrogates, including in \cite{Yu2015b,yu:tel-01514162,Berman2018a}.  \cite{arXiv:1809.00593} claims that these results are false by challenging the correctness of \cite[Proposition~11]{Yu2015b}.  We demonstrate here that \cite{arXiv:1809.00593} makes an erroneous claim and maintain that our results are correct.

\section{Supermodularity of the Jaccard index set functions}

\begin{definition}[Set function \cite{Schrijver2003combinatorial}]
A set function $\ell$ is a mapping from the power set of a base set $V$ to the reals:
\begin{equation}
\ell : \mathcal{P}(V) \rightarrow \mathbb{R} .
\end{equation}
\end{definition}

\begin{definition}[Submodular set function \cite{fujishige2005submodular}]
  A set function $f$ is said to be submodular if for all $A \subseteq B \subset V$ and $x\in V\setminus B$,
  \begin{equation}
    f(A\cup \{x\}) -f(A) \geq f(B \cup \{x\}) - f(B) .
  \end{equation}
\end{definition}



In the following, we fix some arbitrary subset $G \subseteq V$ (the ground truth) and define $f(A) = \frac{|G \cap A|}{|G \cup A|}$ the Jaccard index for a given prediction $A$.

\begin{proposition}\thlabel{prop:JaccardNonSubmodular}
$f(A) = \frac{|G \cap A|}{|G \cup A|}$ is in general neither submodular nor supermodular (the negative of a submodular function).
\end{proposition}
\begin{proof}
Let $A \subset B \subset V$ and $x\in V\setminus B$.  
We may consider two cases: \begin{inparaenum}[(i)] \item $x\in G$ or \item $x \notin G$\end{inparaenum}.

Case (i): $x \in G$.
\begin{equation}
f(A \cup \{x\}) - f(A) - \left( f(B \cup \{x\}) - f(B) \right) = \frac{1}{|G \cup A|} - \frac{1}{|G \cup B|} > 0 ,
\end{equation}
which indicates that $f$ is not supermodular.

Case (ii): $x \notin G$. 
Consider the case that $|G \cap B| = |G\cap A| = 1$ and $|B| = |A|+1$.
\begin{align}
f(A \cup \{x\}) &- f(A) - \left( f(B \cup \{x\}) - f(B) \right) \nonumber \\ =& \frac{|G \cap A|}{|G \cup A| +1} - \frac{|G \cap A|}{|G \cup A|} - \left( \frac{|G \cap B|}{|G \cup B| +1} - \frac{|G \cap B|}{|G \cup B|} \right) \\ =& 
\frac{1}{|G \cup A| +1} - \frac{1}{|G \cup A|} - \left( \frac{1}{|G \cup A| +2} - \frac{1}{|G \cup A|+1} \right) < 0
, 
\end{align}
which indicates that $f$ is not submodular.
\end{proof}

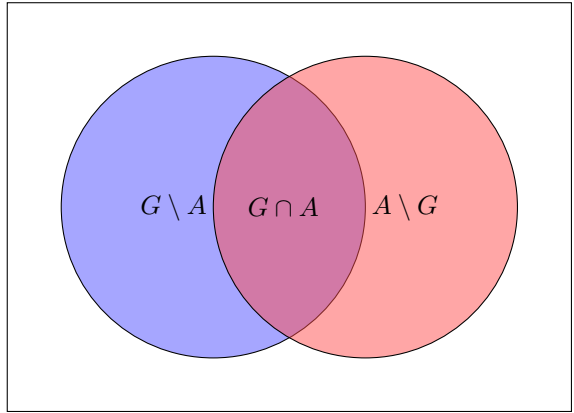
\begin{center}
\begin{tikzpicture}
\def\radius{2cm}
\def\mycolorbox#1{\textcolor{#1}{\rule{2ex}{2ex}}}
\colorlet{colori}{blue!70}
\colorlet{colorii}{red!70}
\coordinate (ceni);
\coordinate[xshift=\radius] (cenii);
\draw[fill=colori,fill opacity=0.5] (ceni) circle (\radius);
\draw[fill=colorii,fill opacity=0.5] (cenii) circle (\radius);
\draw  ([xshift=-20pt,yshift=20pt]current bounding box.north west) 
  rectangle ([xshift=20pt,yshift=-20pt]current bounding box.south east);
\node[xshift=-.5\radius] at (ceni) {$G\setminus A$};
\node[xshift=.5\radius] at (cenii) {$A\setminus G$};
\node[xshift=.9\radius] at (ceni) {$G\cap A$};
\end{tikzpicture}
\end{center}

We define $M := (A \setminus G) \cup (G \setminus A) $, the symmetric difference, or \emph{set of mispredictions} in our context, between $A$ and $G$ (often denoted $A \triangle G$). 
We may uniquely recover $A$ from $M$ by $A = M \triangle G$. 
Let 
\begin{align}
g(M) := f(A) = f(M\triangle G) 
= \frac{|G \setminus M|}{|G \cup M|}
\end{align}
be the Jaccard index as a function of the set of mispredictions. 
\begin{proposition}[Proposition~11~\cite{Yu2015b}]\thlabel{prop:TransformedJaccardSupermodular}
$g(M) = \frac{|G \setminus M|}{|G \cup M|}$ is supermodular and its negative is therefore submodular.
\end{proposition}
\begin{proof}
Let $M \subset N \subset V$ and $x\in V\setminus N$. The following holds:
\begin{align}
|G \setminus N| \leq  |G \setminus M| 
\leq  |G| 
\leq |G\cup M| 
\leq  |G \cup N| . \label{eq:GcupAleqGcupB}
\end{align}

Case (i): $x\in G$.
\begin{equation}
g(M \cup \{x\}) - g(M) - \left( g(N \cup \{x\}) - g(N) \right) = -\frac{1}{|G\cup M|} + \frac{1}{|G \cup N|} \leq 0 .
\end{equation}

Case (ii): $x\notin G$.
\begin{eqnarray}
g(M \cup \{x\}) - g(M) - \left( g(N \cup \{x\}) - g(N) \right) = \frac{|G \setminus M|}{|G \cup M|+1} - \frac{|G \setminus M|}{|G \cup M|}  -  \frac{|G \setminus N|}{|G \cup N|+1} + \frac{|G \setminus N|}{|G \cup N|} 
\leq 0,
\end{eqnarray}
where the inequalities arise by application of \eqref{eq:GcupAleqGcupB}.
\end{proof}
We note that other authors have correctly studied submodular functions of the symmetric difference, e.g.~\cite{NIPS2015_5741,DBLP:journals/corr/abs-1712-08721}.

\section{Conclusion}

In \cite{arXiv:1809.00593} the set functions in \thref{prop:JaccardNonSubmodular} and \thref{prop:TransformedJaccardSupermodular} are conflated leading to the incorrect conclusion that \cite[Proposition~11]{Yu2015b} (which coincides with \thref{prop:TransformedJaccardSupermodular}) is false.  \cite{arXiv:1809.00593} furthermore claims that \cite{Berman2018a} therefore contains errors, but we see that this deduction is based on a false premise.  \cite[Equation~(9)]{Yu2015b} and \cite[Equation~(5)]{Berman2018a} explicitly indicate the correct construction of the set function.  We remain confident in the correctness of \cite{Yu2015b,Berman2018a} and other recent work that builds on \cite{Yu2015b} including \cite{Berman2018b}.

\bibliographystyle{plain}
\bibliography{biblio}

\end{document}